\documentclass{article}

\usepackage{microtype}
\usepackage{graphicx}
\usepackage{subcaption}
\usepackage{booktabs} 

\usepackage{hyperref}

\usepackage[preprint]{icml2026}

\usepackage{amsmath}
\usepackage{amssymb}
\usepackage{mathtools}
\usepackage{amsthm}

\usepackage[capitalize,noabbrev]{cleveref}

\theoremstyle{plain}
\newtheorem{theorem}{Theorem}[section]
\newtheorem{proposition}[theorem]{Proposition}
\newtheorem{lemma}[theorem]{Lemma}
\newtheorem{corollary}[theorem]{Corollary}
\theoremstyle{definition}

\theoremstyle{remark}
\newtheorem{remark}[theorem]{Remark}

\usepackage[textsize=tiny]{todonotes}

\icmltitlerunning{Perplexity Cannot Always Tell Right from Wrong}

\begin{document}

\twocolumn[
  \icmltitle{Perplexity Cannot Always Tell Right from Wrong}



  \icmlsetsymbol{equal}{*}

  \begin{icmlauthorlist}
    \icmlauthor{Petar Veli\v{c}kovi\'{c}}{gdm}
    \icmlauthor{Federico Barbero}{gdm}
    \icmlauthor{Christos Perivolaropoulos}{gdm}
    \icmlauthor{Simon Osindero}{gdm}
    \icmlauthor{Razvan Pascanu}{gdm}
  \end{icmlauthorlist}

  \icmlaffiliation{gdm}{Google DeepMind}

  \icmlcorrespondingauthor{Petar Veli\v{c}kovi\'{c}}{petarv@google.com}

  \icmlkeywords{Machine Learning, ICML}

  \vskip 0.3in
]



\printAffiliationsAndNotice{}  

\begin{abstract}
   Perplexity---a function measuring a model's overall level of ``surprise'' when encountering a particular output---has gained significant traction in recent years, both as a loss function and as a simple-to-compute metric of model quality. Prior studies have pointed out several limitations of perplexity, often from an empirical manner. Here we leverage recent results on Transformer continuity to show in a rigorous manner how perplexity may be an unsuitable metric for model selection. Specifically, we prove that, if there is \emph{any} sequence that a compact decoder-only Transformer model predicts accurately and confidently---a necessary pre-requisite for strong generalisation---it must imply existence of another sequence with very low perplexity, but not predicted correctly by that same model. Further, by analytically studying iso-perplexity plots, we find that perplexity will not always select for the more accurate model---rather, any increase in model confidence must be accompanied by a commensurate rise in accuracy for the new model to be selected.
\end{abstract}

\section{Introduction}

Perplexity is a measure of a model's ``surprise'' when observing ground-truth data; assuming a model's output distribution, $Q$, over a space of classes, $\mathcal{C}$, used to approximate a ground-truth distribution, $P$, it can be expressed as
$
    \exp\sum_{k\in\mathcal{C}}- p_k\log q_k,$
and, since it is easy to compute over any classification task (inlcuding tokenised data), it has become a popular function for evaluating sequential machine learning models when no other performance metric is readily computable. However, even though it is simple to use and interpret, in this paper, we provide novel evidence that perplexity should not be blindly trusted as a model selection objective. 

This is a result that has been informally observed in several venues (\citet{fang2025what, hu2024can,hsieh2024ruler}). Intuitively, it stems from the fact that perplexity encodes both confidence in prediction as well as correctness. Subsequently, a model with a lower accuracy (e.g. more answer tokens predicted incorrectly in the context of language models), but with better-calibrated confidence, may have lower perplexity---and thus be preferred.

\begin{figure}[t]\includegraphics[width=0.979\linewidth]{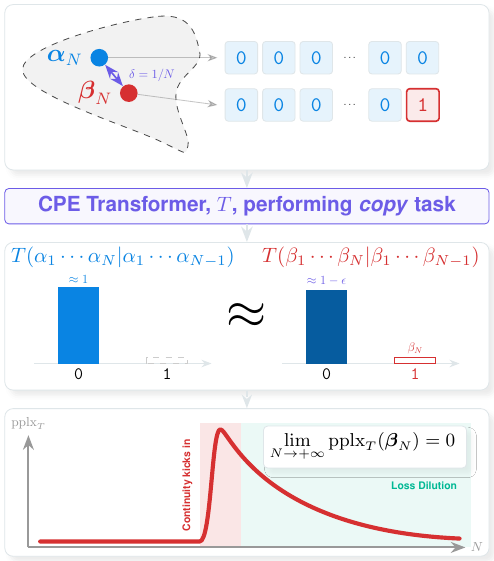}
    \caption{Using the continuity result of \citet{pasten2025continuity}, we show that, if a (compact) Transformer $T$ is confident in copying \emph{any} long enough sequence $\boldsymbol\alpha_N$, then there must exist $\boldsymbol\beta_N$ which $T$ fails to copy, yet, log-perplexity will tend to zero as $N$ grows.}
    \label{fig:placeholder}
\end{figure}

Making use of recent important theoretical developments \citep{pasten2025continuity}, in this work 
we aim to provide a rigorous treatment of some of these observations, along with a detailed analytic account of how perplexity drives unfavourable ``offsets'' between confident and performant models in perplexity's decision space. In doing so, we will prove that any confident model will necessarily introduce inputs which it will likely get wrong, but at a negligible effect to the perplexity measure (see Figure \ref{fig:placeholder}).

Specifically, our work makes the following contributions, all pertaining to the damaging effect of \emph{miscalibrated confidence} on perplexity's power as a model selector:
\begin{enumerate}
\item We prove that, for a wide class of decoder-only Transformer-based language models, should the model be highly confident and correct on a sufficiently long input sequence, this must imply existence of \emph{another} input where the model's prediction is \emph{wrong}, yet the log-perplexity of that prediction approaches \emph{zero}. 
\item We empirically validate this observation by studying \emph{bitstring copy tasks}, both for a custom trained decoder-only Transformer over a small vocabulary, and the Gemma 3 4B large language model \citep{team2025gemma}.
\item Under certain assumptions on homogeneity of confidence, we study \emph{iso-perplexity curves} in the confidence/accuracy space. These curves reveal clear ``unfavourable regions'' where a model gets too confident to justify its own accuracy improvement, and would not be selected against many weaker models.
\item By tracking many checkpoints throughout the training lifecycle of a language model for \emph{parity} prediction, we assess the frequency and circumstances under which models end up on the ``wrong side'' of the iso-perplexity curve in practice. We find that a key driver of undesirable model selection outcomes are \emph{distribution shifts}.
\end{enumerate}

\section{Related work}


While perplexity (or similarly log-likelihood) is the standard metric for evaluating language models, it has long been known in the broader generative model literature that likelihood does not necessarily correlate with sample quality. For instance, \citet{theis2015note} showed that good likelihood scores can be achieved by models that generate poor samples, and conversely, high-quality generators can yield poor likelihoods. \citet{nalisnick2018do} showed that VAEs or flow-based models can assign higher likelihood to images that are \emph{outside} the training distribution and that therefore do not represent the training data.

In the context of language models, \citet{holtzman2019curious} showed a similar disconnect between optimising for likelihood and the generation of high-quality samples. In particular, they show that decoding using Nucleus Sampling leads to better generations (with lower likelihoods) compared to likelihood-maximising approaches such as beam search. 

\paragraph{Failures of Perplexity in Long-Context}
\citet{fang2025what} argue that using perplexity as a metric in long-context is often misleading because  useful signal may vanish when averaging perplexity over thousands of tokens. Their work champions the view that the \emph{aggregation} method is the culprit. Our work rigorously proves results related to this observation, while also extending to claim that there is a detrimental, asymmetric relationship between accuracy and model confidence, which complicates the story further. 

We highlight that this has been alluded to by other work. \citet[Figure 5]{gelberg2025extending} showed that models can maintain low perplexity even when relevant information is strictly unreachable, which seems to explain the effectiveness of the popular context extension method YaRN \citep{peng2023yarn}. Similarly, \citet{liu2024lost} and \citet{hsieh2024ruler} have shown that models often fail to retrieve information `lost in the middle' of a prompt, despite achieving low overall perplexity scores on those same documents. These findings suggest that perplexity is not necessarily aligned with model performance, especially in long-context regimes.

\paragraph{Confidence and Calibration}
A core component of our analysis is the role of model confidence. \citet{guo2017calibration} famously showed that modern neural networks tend to be miscalibrated and overconfident. In the LLM era, while some argue models are generally calibrated \citep{kadavath2022language}, the incentive to minimise perplexity encourages `confident' predictions \emph{in-distribution}. Our analysis studies how these training dynamics allow models to trade accuracy for confidence, creating `unfavourable regions' where a `confident but wrong model' achieves a better perplexity score than a hesitant but more accurate one.

\paragraph{Theoretical Results on Transformers}
Our work relies on recent theoretical works regarding the limitations of the Transformer architecture. \citet{barbero2024transformers} identified the phenomenon of representation collapse in decoder-only Transformers. Extending this, \citet{pasten2025continuity} proved the existence of a `concentration' of infinite sequence collections, such that decoder-only LLMs (under reasonable assumptions on positional embeddings) can model exactly one sequence in each collection. We leverage this continuity result to provide a proof of why perplexity fails: specifically, we show that the existence of a long enough sequence the model predicts accurately and confidently implies the existence of another sequence with very low perplexity that the model still fails to predict over.






\section{Log-perplexity of wrong next-token predictors can arbitrarily approach zero}

For the specific case of autoregressive models trained on next-token prediction (such as large language models), we can recombine a few previous results to theoretically strengthen the empirical finding of \citet{fang2025what}.

\subsection{Preliminaries}

As a clean proxy to the points we are going to make, throughout this section we will focus on a task where both perplexity and correctness are easy to define. Specifically, we study the \textbf{bitstring copy task}: a language model is provided a sequence of bits followed by a unique ``stop'' symbol, $\mathtt{|}$, after which it needs to reproduce the given sequence of bits exactly. For example, given $\mathtt{01010|}$, the model needs to output $\mathtt{01010}$. The model's vocabulary is hence made up of only three symbols: $\mathtt{0}$, $\mathtt{1}$ and $\mathtt{|}$. It is well known that copying is tricky for modern LLMs to learn robustly \citep{barbero2024transformers}, making it an ideal candidate for our study.

Secondly, all of our results rely on the assumption that our language model, $T$, is a decoder-only Transformer with compact position embeddings (CPE). This is necessary in order for the key results of \citet{pasten2025continuity} to apply, and is generally true for the majority of positional embeddings in common use today, such as RoPE \citep{su2024roformer}. We denote the output probability distribution of $T$ as:
\begin{equation}
    T(\mathbf{x})(y) = P_T(y\ |\ \mathbf{x}),
\end{equation}
the probability of emitting symbol $y$ given input prompt $\mathbf{x}$.

\subsection{Deterministic sampling}

In order to make robust claims about a model's accuracy and perplexity, we need to assume it will behave \textbf{deterministically} across all possible input prompts. We hence assume that its outputs are sampled via \emph{greedy decoding}:
\begin{equation}
     T_!(\mathbf{x})= \arg\max_{s\in\{\mathtt{0}, \mathtt{1}\}} T(\mathbf{x})(s)
\end{equation}
assuming all ties are broken consistently, e.g. by always choosing $\mathtt{0}$ in such cases.

In this regime, we will always measure the \textbf{log-perplexity} of the language model, $T$, on the length-$n$ input bitstring $\mathbf{b}\in[\mathtt0, \mathtt1]^n$, defined as follows:
\begin{equation}
    \mathrm{pplx}_T(\mathbf{b}) = -\frac{1}{n}\sum_{k=1}^{n} \log T(b_1\cdots b_n|o_1\cdots o_{k-1})(b_k),
\end{equation}
where the symbols $o_i$ are sampled deterministically:
\begin{equation}
    o_1 = T_!(b_1\cdots b_n|) \qquad o_i = T_!(b_1\cdots b_n|o_1\cdots o_{i-1}).
\end{equation}
This aligns well with the model's loss function, and it is monotonically related to the perplexity. 

Finally, we assume that the model performs all of its computations with appropriate numerical protection, meaning that the obtained values of $\log T(\mathbf{x})(y)$ will never diverge to $-\infty$, and remain bounded by $\log T(\mathbf{x})(y)\ge \log\varepsilon$ for some $\varepsilon > 0$.

\begin{lemma}[Perplexity convergence]
    \label{lemma:pplx-convergence}
    Let $T$ be a decoder-only Transformer with compact position embeddings (CPE), as defined by \citet{pasten2025continuity}. Assume $T$ is trained to perform a copy task over bitstrings, and it samples outputs by greedy decoding.
    
    Let $\boldsymbol\alpha = \alpha_1\alpha_2\cdots\alpha_n\cdots$ be an infinite bitstring. Assume $T$ is capable of correctly copying every finite prefix of $\boldsymbol\alpha$; that is, there is an $\epsilon > 0$ such that, for all $n\in\mathbb{N}$ and $1\leq k\leq n$:\begin{equation}\label{eqn:step}T(\alpha_1\cdots\alpha_{n}|\alpha_1\cdots\alpha_{k-1})(\alpha_k) > 1/2 + \epsilon.
    \end{equation}
    Then, for every $\xi > 0$, there must exist $n'\in\mathbb{N}$ such that, for all prefixes $\boldsymbol\alpha_N = \alpha_1\alpha_2\cdots\alpha_N$ with $N\geq n'$, there is a bitstring $\boldsymbol{\beta}_N$ such that $|\mathrm{pplx}_T(\boldsymbol{\alpha}_N) - \mathrm{pplx}_T(\boldsymbol{\beta}_N)| < \xi$, and $\boldsymbol\beta_N$ is \textbf{not} correctly copied by $T$.
\end{lemma}
Armed with this result (proved in Appendix \ref{app:lem31}), we can now introduce an assumption of $T$ having a certain (high) confidence in copying $\boldsymbol{\alpha}_N$, which will shortly bring us to one of our key results.
\begin{proposition}[Collapsing confidence]\label{propbeta}
   Let $T$ be a decoder-only Transformer with compact position embeddings (CPE), as defined by \citet{pasten2025continuity}. Assume $T$ is trained to perform a copy task over bitstrings, and it samples outputs by greedy decoding.

    Let $\boldsymbol\alpha = \alpha_1\alpha_2\cdots\alpha_n\cdots$ be an infinite bitstring. Assume $T$ is capable of correctly copying every finite prefix of $\boldsymbol\alpha$ with \textbf{confidence} $(1-\gamma)$; that is, there is a $0 \le \gamma < 1/2$ such that, for all $n\in\mathbb{N}$ and $1\leq k\leq n$:\begin{equation}T(\alpha_1\cdots\alpha_{n}|\alpha_1\cdots\alpha_{k-1})(\alpha_k) \ge 1-\gamma.
    \end{equation}
    Then, for every $\epsilon > 0$, there must exist $n'\in\mathbb{N}$ such that, for every size $N\ge n'$, there is a bitstring $\boldsymbol{\beta}_N=\beta_1\cdots\beta_N$ such that $\mathrm{pplx}_T(\boldsymbol{\beta}_N) < -\log(1-\gamma) + \epsilon$, and $\boldsymbol\beta_N$ is \textbf{not} correctly copied by $T$.
\end{proposition}
\begin{proof} This result can be derived by applying Lemma \ref{lemma:pplx-convergence}, setting $(\epsilon=\frac{1}{2}-\gamma,\ \xi=\epsilon)$, and remarking that $\mathrm{pplx}_T(\boldsymbol{\alpha}_N) \le -\frac{1}{N}\sum_{k=1}^N\log(1-\gamma)=-\log(1-\gamma).$
\end{proof}

\begin{corollary}
    If there exists any infinite sequence $\boldsymbol{\alpha}$ copied with certainty ($\gamma=0$) by $T$, then there must exist a family of finite sequences $\boldsymbol{\beta}_N$, such that $\lim\limits_{N\to+\infty}\mathrm{pplx}_T(\boldsymbol\beta_N)=0$, and none of the sequences in $\boldsymbol\beta_N$ are correctly copied by $T$.
\end{corollary}

This result demonstrates that, as models get more confident on any input, this necessarily allows for confounding situations where some other inputs get incorrectly processed without a visible impact on perplexity.

\subsection{Stochastic sampling}

One important assumption that allowed for this result to be cleanly derived is greedy decoding (i.e. sampling with temperature $\theta=0$). As this setup is less common in contemporary use of decoder-only Transformers, here we briefly remark on the applicability of our theoretical results in the stochastic sampling case. In our context, increasing $\theta$ also increases the likelihood of a ``random bit-flip'' which would lead to incorrect copying of the bitstring $\boldsymbol\alpha_N$.

First, we abstract away the choice of $\theta$ by folding it into $\gamma$:
\begin{remark}
    Let $T(\mathbf{a})(\sigma) = (1-\gamma)$ for an input bitstring $\mathbf{a}$ and bit $\sigma\in\{0, 1\}$. Then, assuming we sample with temperature $\theta > 0$, the sampling probability becomes:
    \begin{equation}
        T_\theta(\mathbf{a})(\sigma) = \frac{(1-\gamma)^{1/\theta}}{(1-\gamma)^{1/\theta} + \gamma^{1/\theta}} = 1-\gamma',
    \end{equation}
    where $\gamma'$ is a function of $\gamma$ and $\theta$. Therefore, varying temperature of a $(1-\gamma)$-confident model may be seen as a model with $\theta=1$ but a different confidence level $(1-\gamma')$. Hence, we may assume $\theta=1$ without loss of generality.
\end{remark}
\begin{figure*}
    \includegraphics[width=\linewidth]{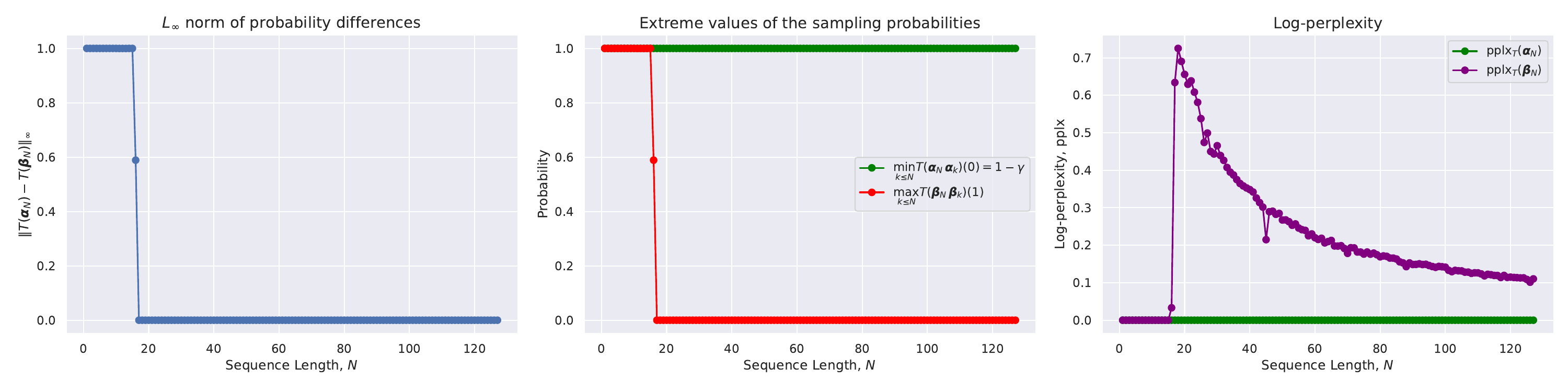}
    \includegraphics[width=\linewidth]{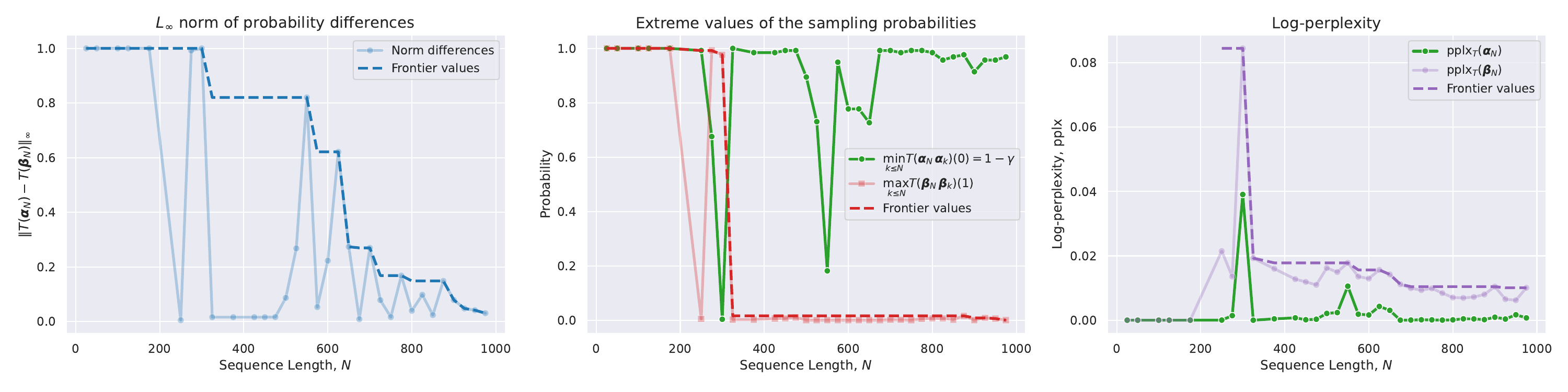}
    \caption{For various sequence lengths, $N$, on the copy task, we compute ({\bf Left}) the $L_\infty$ norm of the difference between the logit distributions across all positions, ({\bf Middle}) the minimal observed probability of predicting $\boldsymbol{\alpha}_k$---our conservative estimate of $1-\gamma$---and the maximal observed probability of predicting $\boldsymbol{\beta}_N$---which can serve as a bound on the probability that the model will copy $\boldsymbol{\beta}_N$ properly. We also plot ({\bf Right}) the log-perplexity for both $\boldsymbol\alpha_N$ and $\boldsymbol{\beta}_N$. This is done both for ({\bf Top}) a toy copy environment where a CPE Transformer is trained on sizes up to 16 bits, and ({\bf Bottom:}) prompting Gemma 3 4B with a copy request.}
    \label{fig:pplxexp}
\end{figure*}
Firstly, we recall a useful result -- Boole's inequality -- which allows us to place a meaningful bound on the probability of bit-flips in a $(1-\gamma$)-confident sampler:
\begin{remark}\label{rem:rem}
    Assume $T$ is capable of correctly copying a length-$N$ bitstring $\boldsymbol\alpha_N$ with confidence $(1-\gamma)$. Then, we can bound the probability of any stochastic copying errors using Boole's inequality (letting $\bar{\alpha}_k=1-\alpha_k$):
    \begin{align*}1 - P_T(\boldsymbol\alpha_N\ |\ \boldsymbol{\alpha}_N|) &\le \sum\limits_{k=1}^NT(\alpha_1\cdots\alpha_N|\alpha_1\cdots\alpha_{k-1})(\bar{\alpha}_k)\\
    &\le \sum_{k=1}^N \gamma = N\gamma.
    \end{align*}
    That is, if $N\gamma\ll 1$, it is unlikely any flips will happen, in which case the baseline sequence $\boldsymbol\alpha_N$ is copied correctly.
\end{remark}
In the stochastic sampling regime, we can leverage the results of \citet{pasten2025continuity} once again, to analyse the probability that $T$ will produce $\boldsymbol\alpha_N$ in response to $\boldsymbol\beta_N$! 
\begin{proposition}
    Assume $T$ is capable of correctly copying every finite prefix of $\boldsymbol\alpha$ with confidence $(1-\gamma)$. Then, for every $\epsilon > 0$ there is an $n'\in\mathbb{N}$ such that, for every size $N\ge n'$, under stochastic output sampling,
    \begin{equation}
        1 - P_T(\boldsymbol{\alpha}_N\ |\ \boldsymbol\beta_N|) \leq N(\gamma + \epsilon),
    \end{equation}
    where $\boldsymbol\beta_N$ is derived as in Proposition \ref{propbeta}.
\end{proposition}
\begin{proof}[Informal proof]
    Follow a similar argument to Remark \ref{rem:rem}, but this time consider $N > \lceil1/\delta\rceil$, where $\delta>0$ is the continuity condition for $\epsilon$, and leverage continuity.
\end{proof}
This result implies that there are three possible outcomes in the stochastic sampling scenario (where $\epsilon_N$ is the smallest value of $\epsilon$ attainable at size $N$):

$N\gamma\ll 1, N\epsilon_N\ll 1$: Corresponds well to our greedy-decoding analysis. The model is confident enough to copy the baseline sequence $\boldsymbol\alpha_N$ with high probability, but it's too tethered to $\boldsymbol\alpha_N$ (due to continuity). Therefore it will fail to copy $\boldsymbol\beta_N$ with high probability (most likely producing $\boldsymbol\alpha_N$).

$N\gamma\ll 1, N\epsilon_N\not\ll 1$: The model copies $\boldsymbol\alpha_N$ with high probability, and the sequence is not long enough for our theory to apply. In this case, we are unable to make concrete claims about the model's behaviour on $\boldsymbol\beta_N$.

$N\gamma\not\ll 1$: The model is not confident enough to reliably copy the baseline sequence $\boldsymbol\alpha_N$, and due to continuity, it will likely fail to copy $\boldsymbol\beta_N$ in the same way.

\subsection{Implications on learning dynamics}
In Lemma \ref{lemma:pplx-convergence}, we showed that the perplexity of an incorrect sequence $\boldsymbol\beta_N$ converges to that of a correct sequence $\boldsymbol\alpha_N$ within a margin $\epsilon$. We now show that this has learnability implications. In particular, as the loss on $\boldsymbol\alpha_N$ goes to $0$, this implies that the loss  on $\boldsymbol\beta_N$ also approaches $0$. Consequently, the training signal for the incorrect sample vanishes, and such a sample cannot be jointly learned (proved in Appendix \ref{app:cor37}).
\begin{corollary}[Vanishing gradients on incorrect samples]
    \label{prop:vanishing-grad}
    Let $\mathcal{L}(\mathbf{x}; T_\theta) = -\frac{1}{M}\sum_{i=1}^M \log T_\theta(\mathbf{x}_{<i})(x_i)$ be the standard autoregressive cross-entropy loss for a CPE decoder-only Transformer with parameters $\theta$. 
    
    Assume that for the sequence $\boldsymbol\alpha_N$, the model achieves a perfect loss, i.e. $\mathcal{L}(\boldsymbol\alpha_N; T_\theta) \to 0$ as $N \to +\infty$. Under the conditions of Proposition \ref{propbeta}, for the sequence $\boldsymbol\beta_N$ which is \textbf{not} correctly copied by $T_\theta$, the gradient of the loss with respect to $\theta$ vanishes:
    \begin{equation}
        \lim_{N \to +\infty} \|\nabla_\theta\mathcal{L}(\boldsymbol\beta_N; T_\theta)\| = 0.
    \end{equation}
\end{corollary}
\subsection{Empirical analysis}

We attempt to validate our theoretical results in Figure \ref{fig:pplxexp}, both when pre-training a CPE Transformer on solely the copy task, and on a larger, general Gemma 3 4B model, setting $\boldsymbol\alpha_N=00\cdots00$ and $\boldsymbol\beta_N=00\cdots01$. Our observations match our expectations: continuity holds in both regimes, with the gap between the probability distributions on $\boldsymbol\alpha_N$ and $\boldsymbol\beta_N$ diminishing with increasing $N$. Further, the probability of continuing $\boldsymbol{\alpha}_N$ remains high and stable for most input sizes, whereas the probability of successfully continuing $\boldsymbol\beta_N$ collapses. All the while, (log-)perplexity indeed gets iteratively closer between the two sequences.

One important caveat with these results is the observed noisy patterns in the Gemma 3 4B experiments. This is due to the fact that, unlike the clear-cut bitstring vocabulary of our theoretical setup, Gemma 3 has a much larger set of possible tokens---and especially due to their failure to count \citep{barbero2024transformers}, on certain occasions the model attempts to prematurely predict newline characters and end-of-turn characters. These both cause issues with the computed probability distribution and perplexities, but they do not affect the overall trends of the relevant metrics collapsing, which we visualised using dashed lines.

Almost none of the results derived so far are specific to perplexity's pointwise form---they, instead, mainly rely on the \emph{averaging} process of the equation. As such, one might be tempted to see this as further evidence of \citet{fang2025what}'s claim that the perplexity function itself might not be inherently problematic---just the way it's aggregated. Furthermore, the LongPPL replacement for perplexity which is proposed in \citet{fang2025what} would not necessarily suffer from the smoothing effects we identify here, as it would significantly shorten the number of tokens for which the metric is computed. That said, we believe there \emph{are} inherent issues in the perplexity function beyond how it's averaged, and this motivates us to study a \emph{pointwise} setup with only one output, but placing important emphasis on the model confidence values.

\section{An analytic view into confidence}

\begin{figure*}
    \includegraphics[width=0.5\linewidth]{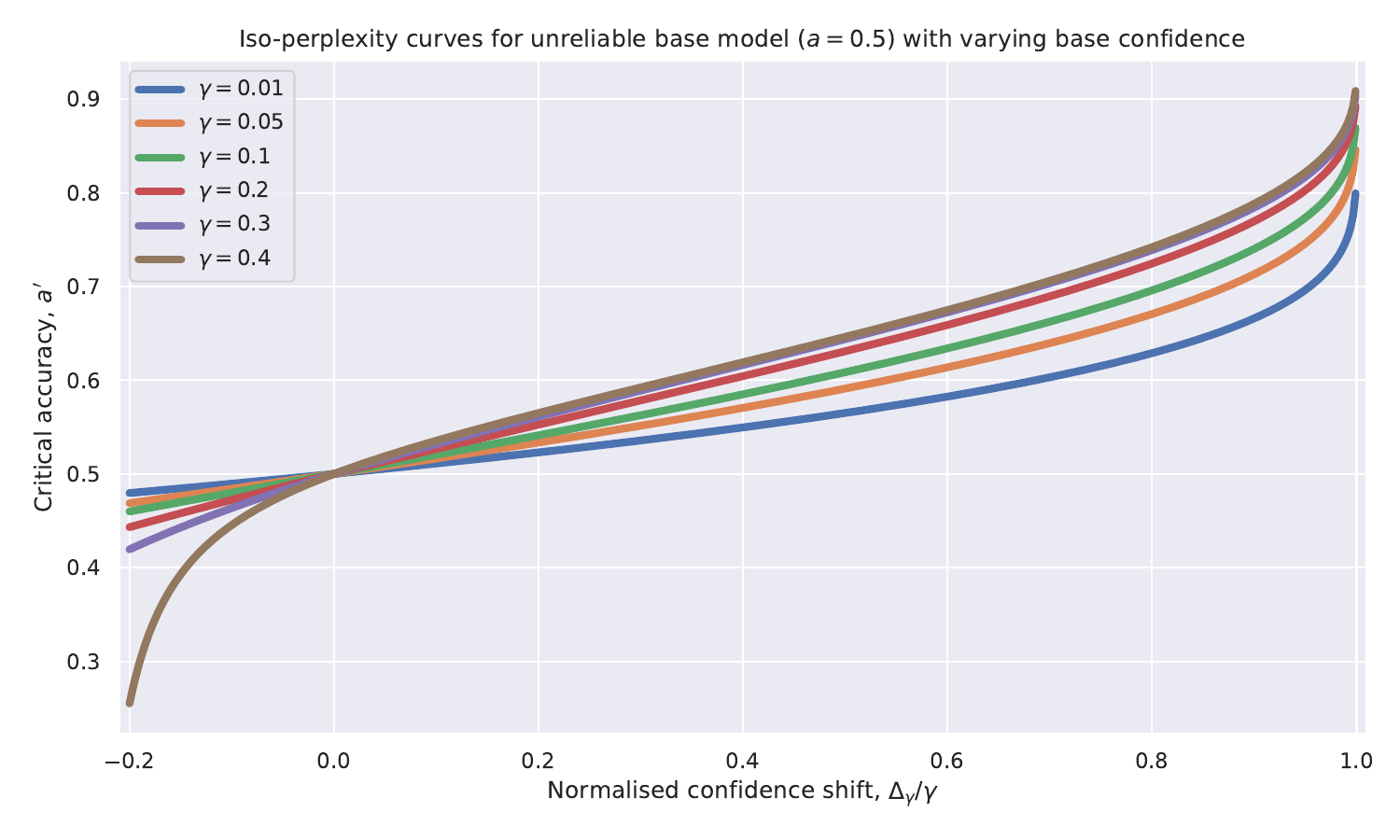} 
    \includegraphics[width=0.5\linewidth]{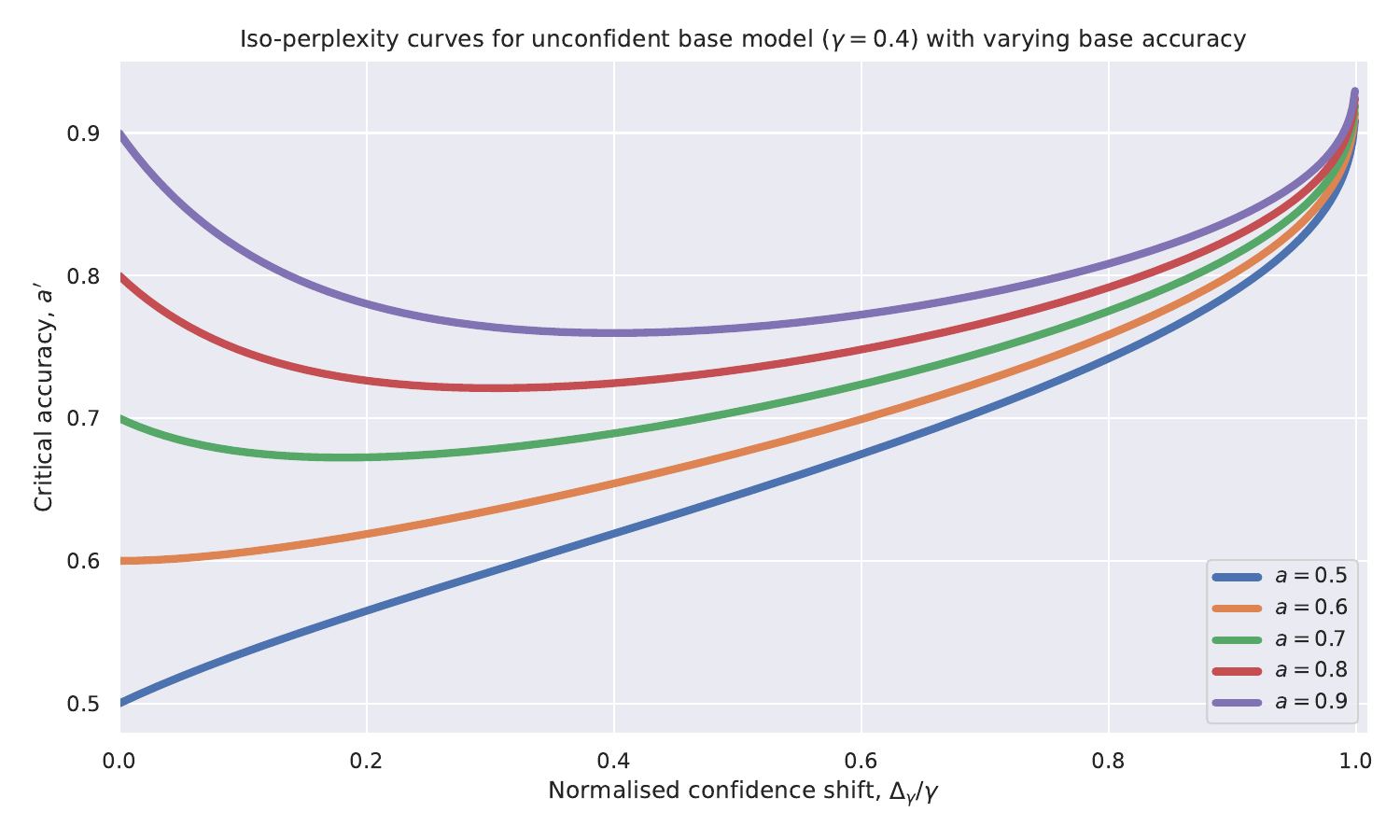}
    \caption{{\bf Left:} Iso-perplexity curves for the setting with an unreliable base model ($a=0.5$) for varying choices of confidence ($1-\gamma$). {\bf Right:} Iso-perplexity curves for an unconfident base model ($\gamma=0.4$) for varying choice of base accuracy $a$,.} \label{fig:isoppl}
\end{figure*}

From the theoretical and empirical analysis we presented so far, one variable that clearly stands out is \emph{confidence}. In systems relying on stochastic sampling, high confidence (low $\gamma$) is important for them to generalise mechanistically---as $N\gamma$ needs to be sufficiently small to attenuate the likelihood of failures over long ranges $N$. However, our theory implies that any high-confidence prediction in CPE Transformers not only opens the door to guaranteed failures elsewhere, it does so in a way that perplexity may not be able to detect the failure. In what follows, we attempt to answer: Can we establish a more general connection between the level of confidence a model has and its predictive power, in a way that reveals how predictable that power is via perplexity?

The answer is affirmative---in that, whenever confidence of a model increases, the model needs to supplement that confidence with a sufficient boost in predictive power---otherwise, perplexity will be unable to recognise this jump in confidence as positive. Specifically, with the right set of initial assumptions, it is possible to \emph{analytically} solve for the ``critical accuracy'' needed to justify increased confidence.

\subsection{Preliminaries}

In order to be able to analytically manipulate the expressions we care about, it is important to make simplifying assumptions that will allow us to abstract away our model's \emph{confidence} ($1-\gamma$) and \emph{accuracy} ($a$) as scalar variables in $[0, 1]$. Further, there must be a simple way to relate those scalar variables to the model's log-perplexity, $\mathrm{pplx}$.

The specific framework we assume which captures this idea well is a binary classification problem, where the model \emph{always} makes its decisions with identical confidence ($1-\gamma$). This means that we can express the log-perplexity over a dataset with accuracy $a$ as:
\begin{equation}\label{eqn:pplxmodel}
    \mathrm{pplx}_{a,\gamma} = - a\log (1-\gamma) - (1-a)\log \gamma
\end{equation}

\subsection{Iso-perplexity curves}

\begin{figure*}
    \centering
    \includegraphics[width=\linewidth]{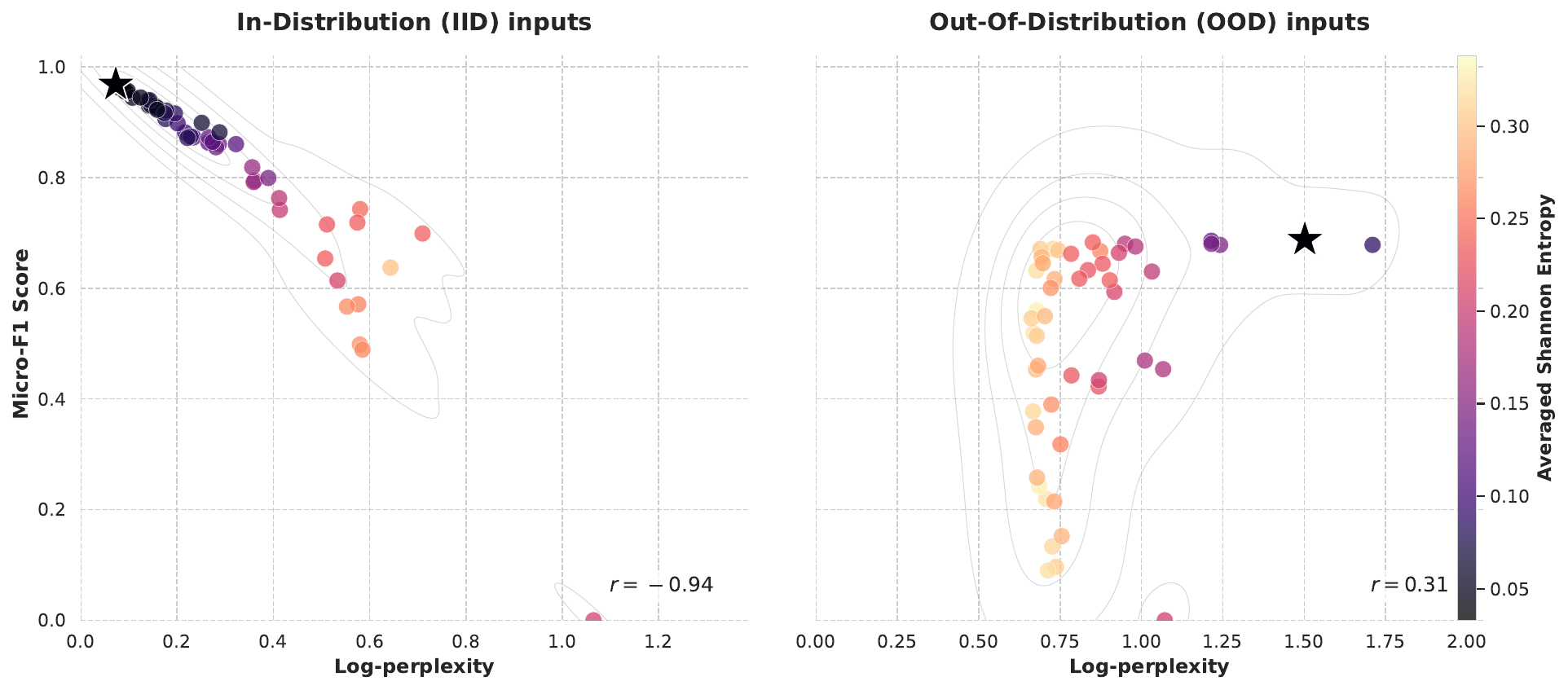}
    \caption{Scatter plots of micro-$\mathrm{F}_1$ scores against log-perplexities, $L$, for various checkpoints of a Transformer model trained on the Parity problem, as specified by \citet{vitvitskyi2025makes}, for both in-distribution ({\bf Left}) and out-of-distribution ({\bf Right}) held-out data. We also colour-code the checkpoints by their averaged Shannon entropy, $\bar{H}$, provide the Pearson correlation coefficient, $r$, and highlight the point with the highest accuracy by using a star (also colour-coded by entropy).}
    \label{fig:parity}
\end{figure*}

Now, we consider a setting where the model gets \emph{more confident} by $\Delta_\gamma\in[0, \gamma]$; that is, its confidence when correct jumps to $1-\gamma + \Delta_\gamma$, and its confidence in the correct answer when wrong drops, symmetrically, to $\gamma-\Delta_\gamma$. If the accuracy doesn't change, the first term of $\mathrm{pplx}$ will decrease while the second will increase. 

But, even though we symmetrically altered the confidence by $\Delta_\gamma$, the change in these two terms is \emph{not} symmetrical, as the $\log$ function has a significantly higher rate of change when its input is close to $0$ compared to being close to $1$. As such, if we keep accuracy the same when increasing confidence, this will in many cases \emph{increase} perplexity. 

Accordingly, a sufficient \emph{rise in accuracy}, $a'$, is needed to compensate for this increase in confidence, if we wish perplexity to recognise this improvement in model accuracy.

This ``critical point'' in accuracy happens when the perplexities of the old and new model become equal:
\begin{align*}-a\log(1-\gamma)&-(1-a)\log\gamma \\= -a'\log(1-\gamma+\Delta_\gamma) & -(1-a')\log(\gamma-\Delta_\gamma)\end{align*}
Rearranging the terms, we obtain:
\begin{align*}
    a(\log\gamma - \log(1-\gamma)) &- \log \gamma \\= a'(\log (\gamma - \Delta_\gamma) - \log(1-\gamma+\Delta_\gamma)) &- \log (\gamma-\Delta_\gamma)
\end{align*}
From here, we can explicitly derive the required accuracy:
\begin{equation}\label{eqn:isoppl}
    a' = \frac{\mathrm{pplx}_{a,\gamma}+\log (\gamma - \Delta_\gamma)}{\log(\gamma-\Delta_\gamma)-\log(1-\gamma+\Delta_\gamma)}
\end{equation}
Note that this function depends on both the initial accuracy $a$ and initial confidence $(1-\gamma)$. Therefore, it gives rise to several types of \emph{iso-perplexity curves}, depending on whether we keep $a$ fixed and vary $\gamma$, or keep $\gamma$ fixed and vary $a$. To illustrate what these curves teach us about the reliability of perplexity as a discriminative metric, we focus on two specific cases here:

\paragraph{Iso-perplexity at $a=0.5$} In this setting, we assume starting from an entirely unreliable model, but varying the starting confidence, $(1-\gamma)$. We plot the \emph{critical accuracy}, $a'$, against the \emph{normalised confidence shift}, $\Delta_\gamma/\gamma$; see Figure \ref{fig:isoppl} (Left). Note that any model falling under the iso-perplexity curve would \emph{not} be selected as improving perplexity, even though its accuracy may be better than the random chance of the first model -- similarly, a model may end up above the iso-perplexity curve even though its accuracy is worse than the baseline. We can make two key observations:

Firstly, for all considered starting confidences, not all better more-confident models will decrease perplexity. The afforded ``breathing room'' for $a'$ tends to be greater (in relative terms) the more confident the base model is -- there is ``less surprise'' when making an already confident model more confident. Still, many confidence shifts require increasing accuracy by over 5--10 percentage points, which is very significant.

Secondly, no matter what the starting confidence, truly extraordinary confidence requires truly extraordinary evidence---as $\Delta_\gamma\to\gamma$, $a'\to1$. Put differently, a perfectly confident model must be perfectly accurate, otherwise it will always be rejected by perplexity.

\paragraph{Iso-perplexity at $\gamma=0.4$} In this setting, we start from an unconfident model, but varying the starting accuracy, $a$---once again plotting critical accuracy against normalised confidence shift. The resulting iso-perplexity curves are in Figure \ref{fig:isoppl} (Right). The key insight that this regime offers is the existence of ``unjustified free lunch'' zones, where the iso-perplexity curves are decreasing for positive $\Delta_\gamma$. If an unconfident model is already sufficiently accurate, it is possible to improve their perplexity just by making them more confident -- even if this leads to significant drops in accuracy ($a' < a$).

It is evident that there exists a rather non-negligible space under the iso-perplexity curve with $a' > a$, as well as below it with $a' < a$; it describes a significant family of models that would not be selected by perplexity, in spite of being better predictors than their baseline. We hypothesise this might have implications in many relevant regimes of AI deployment where accuracy cannot be easily measured, \emph{especially} as the model needs to predict outside of its training distribution, which often requires higher confidence.

\subsection{How often are we on the wrong iso-perplexity side?}

Having exposed that there exist very clear regions of the model confidence/accuracy space where perplexity would not select the more accurate model, what remains to be seen is to what extent will this occur in practice.

One might hypothesise that a model is particularly vulnerable to such failures when the inputs stray \emph{out-of-distribution} (OOD) compared to data the model was exposed to during training. Indeed, as we saw with the copy task example, models might \emph{need} to maintain a low value of $\gamma$ in order to even be able to process their baseline sequences properly. However, higher confidence also implies that when any failures do occur, they will be especially painful to perplexity.

Note that our ``identical-confidence'' model described in Equation \ref{eqn:pplxmodel} is substantially \emph{constrained} in order to make iso-perplexity curves analytically derivable---in reality, there may well not be a value of $\gamma$ that fits an observed perplexity/accuracy pair $(L,a)$ over a real dataset. That is, there often may be no $\gamma\in[0, 1]$ such that $L = -a\log(1-\gamma) -(1-a)\log\gamma$. Furthermore, if any models start to get less confident, iso-perplexity curves approach their singularity point at $\gamma=0.5$, at which point it gets hard to see the relevant phase transitions on the confidence/accuracy plot.

For all of the above reasons, we abandon plotting the iso-perplexities here, and instead directly plot the $(L, a)$ pairs we observed via a scatter plot. Whenever $L_1 < L_2$ but $a_1 < a_2$, the model's perplexity cannot discriminate properly between these two points, and we can estimate how often this happens by observing the Pearson correlation coefficient, $r$. In an ideal setting, where the $L$ metric exactly orders accuracies, we would recover $r \approx -1$.

Beyond measuring the frequency of incorrect model selections, we also want to ascertain that these issues can be directly related to model confidence. While we cannot map arbitrary logit collections $\{\log p_i\}_{i=1}^n$ to a fixed value of $\gamma$, we \emph{can} compute a proxy for the model's overall level of certainty by computing the \emph{averaged Shannon entropy}, $\bar{H}= -\frac{1}{n}\sum_{i=1}^n p_i\log p_i$. We can then use this quantity to colour-code the individual models we're studying on the scatter plot---under the assumption that points that will be particular outliers in the OOD setting are the ones where $\bar{H}$ is lower (when the model is on the whole more confident).

\subsection{Parity task setup} When deciding on which problem to choose to study these effects, it is not only desirable for the task to have a natural OOD regime---it should also be seen as ``mechanistically easy but practically hard''. By this, we mean that there is a very clear, simple procedure that generates the ground-truth outputs, yet it is known that reproducing those outputs is hard for contemporary AI systems. The existence of a clear target procedure means that models need to get confident in order to replicate this procedure; the practical hardness means that their confidence will not always be rewarded.

A very good fit is the \textbf{parity} task: given a bitstring, predict the exclusive-or (XOR) of all of its bits (e.g., for $\mathtt{01010}$, predict $\mathtt{0}$; for $\mathtt{11010}$, predict $\mathtt{1}$). Parity is well-understood to be difficult when length-generalising with Transformers, for known theoretical reasons \citep{hahn2020theoretical}, yet the target formula for computing the output is very simple.

We replicate the Transformer training setup for the Parity task from \citet{vitvitskyi2025makes}, reusing the baseline hyperparameters leveraged there, and training the model for $5,000$ gradient steps. Our aim is to appreciate how the model's performance/confidence profile evolves throughout training, and hence, we save many checkpoints of the model throughout training---one for every $100$ steps of gradient descent taken---and evaluate them on held-out in-distribution (IID) and out-of-distribution (OOD) bitstrings in terms of size. In this case, we train on bitstrings of size up to $16$, and consider an OOD distribution of bitstrings of size $128$. 

Overall, this procedure generates a dataset of $(L, \mathrm{F}_1, \bar{H})$ tuples, where $L$ is (log-)perplexity, $\mathrm{F}_1$ is the micro-$\mathrm{F}_1$ score obtained by the model on those sequences, and $\bar{H}$ is the averaged Shannon entropy estimated using the model's individual parity prediction logits across the entire bitstring.

\subsection{Results and Discussion}

We visualise the corresponding scatter plots of the stored checkpoints, along with other useful  data (colour-coding, Pearson correlation) in Figure \ref{fig:parity}. We find that these visualisations provide strong evidence for our hypothesis, by making the following observations over the two data distributions:

\paragraph{Training progression} While in-distribution the model appears to gradually improve its loss and performance, with a corresponding decrease in entropy as the model gets more confident, the same trajectory cannot be observed in the OOD case. Worse yet, the checkpoint with the optimal OOD accuracy is one of the worst in terms of OOD perplexity.

\paragraph{Pearson correlation} The IID evaluations of the checkpoints paint a picture of a model whose perplexity improvements, for the most part, track micro-$\mathrm{F}_1$ score improvements; indeed, with $r=-0.94$, there is a strong anticorrelation between the two variables. No such trend can be observed OOD, in fact, the empirical value of $r$ is \emph{positive} rather than negative. This neatly translates to the likelihood that we have a pair of incorrectly discriminated points with $L_1 < L_2$ but $a_1 < a_2$: it is very high in the OOD regime.

\paragraph{Entropy connection} Lastly, the entropy colour-coding reveals the final piece of the puzzle and matches our hypothesis very well. In-distribution, entropy reduction is a sign of model maturity: the confidence increase follows a clear jump in predictive power and decrease in loss. Out-of-distribution, however, the checkpoints with low entropy can retain predictive power while drastically harming perplexity. In fact, the aforementioned best-performing observed OOD model has one of the lowest entropies in the entire dataset.

All taken together, we can make a clear conclusion: in the right kind of \emph{out-of-distribution} regime, \emph{many} points end up on the \emph{wrong} side of iso-perplexity, and this effect can be directly tied to an \emph{increase in confidence}.

\section{Conclusions}

In their recent important work, \citet{fang2025what} make a clear stance on the issues behind perplexity on long ranges:

\emph{``\dots there is growing evidence that LLMs’ perplexity does not indicate their performance on long-context benchmarks. There are two possible sources of this mismatch: either the log-likelihood-based metric is flawed, or the averaged tokens are not representative enough. In this
work, we champion the latter explanation\dots''}

We provided theoretical evidence in support of the latter source---with all tokens contributing to an averaged loss, this has the potential to lead to \emph{weird} situations, where a model makes confident mistakes on an input, yet its log-perplexity can get arbitrarily close to zero for that input.

However, we also found that the former source cannot be ignored---the perplexity metric itself is inherently skewed, and prone to favouring less confident predictors, \emph{especially} in the long-context settings mentioned above. We found that high model confidence, coupled with a perplexity objective, can be the very reason for being able to construct the weird situations in the former paragraph. We provided additional evidence for this by studying the ample unfavourable regions with respect to \emph{iso-perplexity curves}.

While we do not offer an alternative to perplexity in regimes where accuracy cannot be measured, we hope that our work serves as a useful foundation for exercising appropriate care when using perplexity, as well as offering a few ``diagnostic approaches'' that can help us estimate in which situations one might need to rethink their model selection protocol.

\section*{Acknowledgements}
We would like to thank Alex Vitvitskyi for help with setting up the Parity experiment, and Xiangming Gu and Shakir Mohamed for reviewing the paper prior to submission.

\section*{Impact Statement}

This paper presents work whose goal is to advance the field of Machine
Learning. There are many potential societal consequences of our work, none
which we feel must be specifically highlighted here.


\bibliography{example_paper}
\bibliographystyle{icml2026}

\newpage
\appendix
\onecolumn
\section{Proof of Lemma 3.1.} \label{app:lem31}

\begin{lemma}[3.1.: Perplexity convergence]
    
    Let $T$ be a decoder-only Transformer with compact position embeddings (CPE), as defined by \citet{pasten2025continuity}. Assume $T$ is trained to perform a copy task over bitstrings, and it samples outputs by greedy decoding.
    
    Let $\boldsymbol\alpha = \alpha_1\alpha_2\cdots\alpha_n\cdots$ be an infinite bitstring. Assume $T$ is capable of correctly copying every finite prefix of $\boldsymbol\alpha$; that is, there is an $\epsilon > 0$ such that, for all $n\in\mathbb{N}$ and $1\leq k\leq n$:\begin{equation}T(\alpha_1\cdots\alpha_{n}|\alpha_1\cdots\alpha_{k-1})(\alpha_k) > 1/2 + \epsilon.
    \end{equation}
    Then, for every $\xi > 0$, there must exist $n'\in\mathbb{N}$ such that, for all prefixes $\boldsymbol\alpha_N = \alpha_1\alpha_2\cdots\alpha_N$ with $N\geq n'$, there is a bitstring $\boldsymbol{\beta}_N$ such that $|\mathrm{pplx}_T(\boldsymbol{\alpha}_N) - \mathrm{pplx}_T(\boldsymbol{\beta}_N)| < \xi$, and $\boldsymbol\beta_N$ is \textbf{not} correctly copied by $T$.
\end{lemma}
\begin{proof}
    Since $T$ is a CPE decoder-only Transformer, by \citet{pasten2025continuity}'s continuity theorem, there must exist $\delta > 0$ such that, for any two equal-length inputs $\mathbf{x}$ and $\mathbf{x'}$, if their relativised Hamming distance $d_H(\mathbf{x}, \mathbf{y}) < \delta$ and their last symbol is identical, then $\|T(\mathbf{x})-T(\mathbf{y})\|_\infty\leq \epsilon$. 
    
    Whenever $n_c > \lceil 1/\delta\rceil$, we can find a $\boldsymbol{\beta}_{n_c}$ such that $d_H(\boldsymbol{\alpha}_{n_c}, \boldsymbol{\beta}_{n_c}) < \delta$---simply flip exactly one bit in $\boldsymbol{\alpha}_{n_c}$ at an arbitrary position, $j$. Coupled with Equation \ref{eqn:step}'s assumption, we can deduce \begin{equation}T(\beta_1\cdots\beta_{n_c}|\alpha_1\cdots\alpha_{k-1})(\alpha_k) > 1/2,\end{equation} 
    therefore, \begin{equation}T_!(\beta_1\cdots \beta_{n_c}|) = \alpha_1 \qquad T_!(\beta_1\cdots\beta_{n_c}|\alpha_1\cdots\alpha_{k-1}) = \alpha_k\end{equation} for all $1\le k\le n_c$. That is, $\boldsymbol{\beta}_{n_c}$ is \textbf{not} correctly copied by $T$, and its copying log-perplexity is: \begin{equation}\mathrm{pplx}_T(\boldsymbol{\beta}_{n_c}) = -\frac{1}{n_c}(\log T(\beta_1\cdots \beta_{n_c}|\alpha_1\cdots\alpha_{j-1})(\beta_j) +\sum_{k\ne j} \log T(\beta_1\cdots \beta_{n_c}|\alpha_1\cdots \alpha_{k-1})(\alpha_k)).\end{equation}
    Now, once we observe that we can also, analogously, express \begin{equation}\mathrm{pplx}_T(\boldsymbol{\alpha}_{n_c}) = -\frac{1}{n_c}(\log T(\alpha_1\cdots \alpha_{n_c}|\alpha_1\cdots\alpha_{j-1})(\alpha_j) +\sum_{k\ne j} \log T(\alpha_1\cdots \alpha_{n_c}|\alpha_1\cdots \alpha_{k-1})(\alpha_k)),\end{equation}
    we can match the relevant terms in the summations to obtain $|\mathrm{pplx}_T(\boldsymbol\alpha_{n_c}) - \mathrm{pplx}_T(\boldsymbol\beta_{n_c})|\le -\frac{1}{n_c}(\log (\frac{1}{2}+\epsilon) - \log\varepsilon + (n_c-1)\epsilon)$. By algebraic manipulation of this expression we can conclude that, as long as we choose $n_c > \frac{\epsilon - \log(\frac{1}{2}+\epsilon)+\log\varepsilon}{\xi+\epsilon}$, it will hold that $|\mathrm{pplx}_T(\boldsymbol\alpha_{n_c}) - \mathrm{pplx}_T(\boldsymbol\beta_{n_c})| < \xi$. This implies that we can set \begin{equation}n' = \max\left(\underbrace{\lceil1/\delta\rceil}_\mathrm{continuity}, \underbrace{\frac{\epsilon - \log\left(\frac{1}{2}+\epsilon\right)+\log\varepsilon}{\xi+\epsilon}}_\mathrm{oversmoothing}\right),\end{equation}
    at which point we are guaranteed to obtain both the effects of continuity, misclassifying $\boldsymbol{\beta}_{n'}$, \emph{and} smoothing out the perplexity spike obtained by that misclassification.
\end{proof}

\section{Proof of Corollary 3.7} \label{app:cor37}

\begin{corollary}[3.7.: Vanishing gradients on incorrect samples]
    Let $\mathcal{L}(\mathbf{x}; T_\theta) = -\frac{1}{M}\sum_{i=1}^M \log T_\theta(\mathbf{x}_{<i})(x_i)$ be the standard autoregressive cross-entropy loss for a CPE decoder-only Transformer with parameters $\theta$. 
    
    Assume that for the sequence $\boldsymbol\alpha_N$, the model achieves a perfect loss, i.e. $\mathcal{L}(\boldsymbol\alpha_N; T_\theta) \to 0$ as $N \to +\infty$. Under the conditions of Proposition \ref{propbeta}, for the sequence $\boldsymbol\beta_N$ which is \textbf{not} correctly copied by $T_\theta$, the gradient of the loss with respect to $\theta$ vanishes:
    \begin{equation}
        \lim_{N \to +\infty} \|\nabla_\theta\mathcal{L}(\boldsymbol\beta_N; T_\theta)\| = 0.
    \end{equation}
\end{corollary}
\begin{proof}
The gradient of the cross-entropy loss with respect to parameters $\theta$ is given by 
\begin{equation}\nabla_\theta \mathcal{L} = \frac{1}{N} \sum_{i=1}^N (\mathbf{p}_i - \mathbf{y}_i)^\top J_\theta(\mathbf{x}{<i}),\end{equation}
where $\mathbf{p}_i = T_\theta(\mathbf{x}_{<i})$ is the predicted probability distribution, $\mathbf{y}_i$ is the one-hot target vector for $\boldsymbol\beta_i$, and $J_\theta(\mathbf{x}_{<i}) = \nabla_\theta T_\theta(\mathbf{x}_{<i})$ is the Jacobian of the model logits with respect to the parameters. We can bound the norm of the gradient using the Cauchy-Schwarz inequality:
\begin{equation}\|\nabla_\theta \mathcal{L}\| \leq \frac{1}{N} \sum_{i=1}^N \|\mathbf{p}_i - \mathbf{y}_i\| \|J_\theta(\mathbf{x}{<i})\|.\end{equation}
As we assume that the Transformer is compact, this implies that the norm of the Jacobian is upper bounded by some $K$ (Lipschitz property), i.e. $\sup_\mathbf{x} \|J_\theta(\mathbf{x})\| \leq K$. Furthermore, since the cross-entropy loss is minimised only when the predictive distribution matches the target, convergence in loss implies convergence in
the predicted targets:
\begin{equation}
\lim_{N \to+ \infty} \|\mathbf{p}_i - \mathbf{y}_i\| = 0 \quad \forall i.
\end{equation}
Substituting the bounds back, we achieve the desired result:
\begin{equation}\lim_{N \to +\infty} \|\nabla_\theta \mathcal{L}\| \leq \lim_{N \to+ \infty}  K \frac{1}{N} \sum_{i=1}^N \|\mathbf{p}_i - \mathbf{y}_i\| = 0.\end{equation}
\end{proof}


\end{document}